\documentclass[11pt]{article}
\usepackage[T1]{fontenc}
\usepackage[margin=1in]{geometry}

\usepackage[english]{babel}

\usepackage{amsmath,amssymb,amsfonts,mathtools}
\usepackage{bbm,dsfont}

\usepackage{graphicx}
\usepackage{float}

\usepackage{siunitx}

\usepackage{amsthm}
\newtheorem{theorem}{Theorem}
\newtheorem{proposition}{Proposition}

\newtheorem{assumption}{Assumption}

\newtheorem{remark}{Remark}

\usepackage[colorlinks=true,allcolors=blue]{hyperref}
\usepackage[capitalise]{cleveref}

\usepackage{cite}

\usepackage{xcolor}
\usepackage[normalem]{ulem}
\usepackage{microtype}
\allowdisplaybreaks
\title{\LARGE \bf SLAM as a Stochastic Control Problem with Partial Information: Optimal Solutions and Rigorous Approximations}
\author{Ilir Gusija, Fady Alajaji, and Serdar Y\"uksel%
    \thanks{This work was supported in part by the Natural Sciences and Engineering Research Council of Canada. The authors are with the Department of Mathematics and Statistics,
    Queen's University, Kingston, ON, Canada.
    {\tt\small \{ilir.gusija, fa, yuksel\}@queensu.ca}}}
\begin{document}
\maketitle

\begin{abstract}
Simultaneous localization and mapping (SLAM) is a foundational state estimation problem in robotics in which a robot accurately constructs a map of its environment while also localizing itself within this construction.
We study the active SLAM problem through the lens of optimal stochastic control, thereby recasting it as a decision-making problem under partial information.
After reviewing several commonly studied models, we present a general stochastic control formulation of active SLAM together with a rigorous treatment of motion, sensing, and map representation.
We introduce a new exploration stage cost that encodes the geometry of the state when evaluating information-gathering actions.
This formulation, constructed as a nonstandard partially observable Markov decision process (POMDP), is then analyzed to derive rigorously justified approximate solutions that are near-optimal.
To enable this analysis, the associated regularity conditions are studied under general assumptions that apply to a wide range of robotics applications.
For a particular case, we conduct an extensive numerical study in which standard learning algorithms are used to learn near-optimal policies.
\end{abstract}

\section{Introduction}

\subsection{SLAM Literature Review}\label{sub:Literature Review} 
    Active simultaneous localization and mapping (SLAM) is the problem of controlling
    a robot in an effort toward both accurately estimating its pose
    and constructing a map of its environment \cite{placed2023active}.
    The necessity for simultaneous operations is driven by applications where
    pre-mapping is infeasible due to dynamic environments, resource
    constraints, or inaccessible locations such as underwater, underground, or
    extraterrestrial exploration \cite{stachniss2009robotic}. Although most
    research focuses on estimation, active SLAM techniques have been shown to
    significantly improve exploration efficiency and reduce mapping uncertainty
    \cite{placed2023active,sim2005active}.
    
    By formulating active SLAM as a partially observable Markov decision process
    (POMDP), stochastic control techniques have been applied with some success. 
    Much work has focused on the related problem of belief-space planning
    (BSP), which encompasses path planning under uncertainty, active
    localization, and exploration as special cases \cite{indelman2015planning}.
    However, most BSP approaches rely on Gaussian belief
    representations and local optimization
    without global convergence guarantees
    \cite{platt2010belief,van2012motion,indelman2015planning}.
    Extensions imposing belief uncertainty constraints \cite{rahman2021uncertainty}
    were adapted to active SLAM in \cite{koga2022active}.
    Information-theoretic metrics have proven effective as
    utility functions for active SLAM, particularly \emph{information gain}
    \cite{stachniss2009robotic} and variants based on R\'enyi entropy
    \cite{carrillo2018autonomous} and generalized behavioral entropy
    \cite{suresh2024behave}. A mutual information-based metric was shown to
    provably steer robots toward unexplored areas
    \cite{julian2013mutual}.

    Recently, rigorous approximation and learning methods for POMDPs in general
    state and action spaces have been developed. Saldi et al.
    \cite{saldi2019asymptotic} showed that quantization-based finite models
    yield near-optimal policies under weak continuity conditions.
    Subsequent work on finite memory approximations and Q-learning convergence
    for both schemes is surveyed in \cite{kara2024pomdp}.

    \subsection{Contributions}\label{sub:Contributions} 
    The technical contributions of this paper are as follows.
    In Section~\ref{sec:Preliminaries}, we present a broad formulation of the model.
    In Section~\ref{sec:SLAM as a nonstandard POMDP}, a very general
    POMDP formulation for the SLAM problem is established. Notably in 
    Section~\ref{sub:Cost function}, we introduce a new belief-dependent cost. In
    Section~\ref{sec:Properties of Active SLAM}, we study several regularity
    properties of active SLAM. We
    establish, in particular, weak continuity properties; see Theorems~\ref{thm:cmc}-\ref{thm:weakfeller}. We finish the section by establishing the existence of an optimal policy (see Theorem~\ref{thm:existence}). In Section~\ref{sec:Finite Approximations}, we build on recent advances in finite model approximations for POMDPs
    from \cite{saldi2019asymptotic}, \cite[Section~4]{kara2024pomdp}
    which provide convergence guarantees for discretized belief spaces.
    Under appropriate continuity conditions, these finite approximations
    can achieve provably near-optimal performance, which we show
    for the active SLAM problem in Theorem~\ref{thm:4}. Finally, we provide a simulation case study to illustrate
    our theoretical findings in Section~\ref{sec:Simulation results}.
\section{Preliminaries}\label{sec:Preliminaries} 
In this section, we introduce the general components typical of active SLAM and
related BSP problems, specifically covering the motion model, representation of
the map, and sensor model.

Let $(\mathcal{W}, d_{\mathcal{W}})$ be a compact
Euclidean metric space representing the workspace of the robot, equipped with
the Euclidean metric $d_{\mathcal{W}}(x,y)=\|x-y\|$ and the Lebesgue measure
$\lambda$ restricted to $\mathcal{W}$.
    \subsection{Motion model}\label{sec:Motion model} 
    Let $\mathbb{X}$ be a standard Borel space (i.e., a Borel space associated with a complete separable metric space) that denotes the set of all
    possible poses of the robot. Let $x_{t}\in \mathbb{X}$ denote
    the pose of the robot at $t\ge 0$. The robot pose evolves in time according
    to a controlled stochastic process
    \begin{equation} 
        x_{t+1} = f(x_t, u_t, \xi_t), \quad \xi_t \stackrel{\text{iid}}{\sim} \mu_\xi,
        \label{eq:robot_dynamics} 
    \end{equation}
    where $u_t \in \mathbb{U}$ is the control input and $\xi_t\in \mathbb{R}^d,d>1$
    represents an independent and identically distributed (iid) process noise
    with distribution $\mu_{\xi}$ (written as $\xi_{t}\stackrel{\text{iid}}\sim
    \mu_\xi$).

    \subsection{Modelling the map}\label{sub:Modelling the map} 
    Let $\mathbb{M}$ be a standard Borel space representing the space of maps. The representation of the map varies across different approaches in SLAM, this formulation subsumes these various representations including occupancy maps and landmark-based maps. 
    
    \textbf{Occupancy maps} are represented by an occupied
    region $m\subseteq \mathcal{W}$. We define the map space as the
    family of nonempty compact subsets:
    \begin{equation*}
    \mathbb{M}_{\text{occ}} := \{ m\subseteq \mathcal{W}: m \text{
    compact}, m \ne \emptyset
    \}. 
    \end{equation*}
    Equipped with the Hausdorff metric $d_H$ (or other variants introduced in \cite{baddeley1992errors}), $\mathbb{M}_{\text{occ}}$ is a compact metric space whenever $\mathcal{W}$ is compact \cite{fell1962hausdorff,baddeley1992errors}.

    \textbf{Landmark-based maps} are defined as a finite collection of locations corresponding 
    to landmarks. The map space is defined as   
    $\mathbb{M}_{\ell}\coloneq \mathcal{W}^{ \overline{\ell }}$
    equipped with the Euclidean metric (assuming perfect
    data association and labelled landmarks), where $\overline{\ell }>0$
    represents the number of landmarks.

    The metric space structure of the map is critical in establishing (i)
    continuity properties of the dynamics and observations,
    (ii) continuity of the stage cost (Section~\ref{sub:Cost function}), and (iii) compactness properties.
    
    Let $m_{t}$ denote the state of the map at time $t\ge 0$. The map $m_t$ is
    assumed to be static, so its dynamics are simply described as
    \begin{equation}
        m_{t+1} = m_t,\quad t\ge 0.
        \label{eq:map dynamics}
    \end{equation}

    \subsection{Sensor model}\label{sub:Sensor model} 
    Let $\mathbb{Y}$ be a standard Borel space that denotes the observation space. 
    At each time step $t\ge0$, the robot receives noisy observations
    $y_t \in \mathbb{Y}$ of the environment according to a measurement model 
    \begin{equation} 
        y_t = g(x_{t}, m_{t}, u_{t-1}, \zeta_t), \quad \zeta_t \stackrel{\text{iid}}{\sim} \mu_{\zeta},
        \label{eq:measurement_model} 
    \end{equation}
    where $\zeta_t\in \mathbb{R}^d,d>0$.
    The measurement dynamics stated generally in \eqref{eq:measurement_model}
    vary across sensor modalities. Following the general classification of
    \cite[Chapter~1.2]{barfoot2024state}, robotic sensors
    can be divided into two distinct categories: proprioceptive and exteroceptive.
    
    \textbf{Proprioceptive sensors} measure internal robotic state, are 
    control-dependent, and are independent of the map: $y_t^{\text{prop}} = 
    g^{\text{prop}}(x_t, u_{t-1},v_t)$ \cite[Section~7.4.4]{barfoot2024state}. 
    Examples of proprioceptive sensors include inertial measurement units
    (IMU), gyroscopes, wheel encoders, and joint encoders.

    \textbf{Exteroceptive sensors} measure the state of the external environment 
    and, in contrast to proprioceptive sensors, are independent of the 
    control and a function of the map: $y_t^{\text{ext}} = 
    g^{\text{ext}}(x_t, m_t, \zeta_t)$ \cite[Section~7.4]{barfoot2024state}.
    
    Both sensor classes are subsumed by \eqref{eq:measurement_model} with 
    appropriate dependence structure. 
    For the active SLAM problem, the use of exteroceptive sensors is required while proprioceptive sensors 
    can be seen as complementary by aiding in improving the localization accuracy of the robot.

\section{SLAM as a nonstandard POMDP}\label{sec:SLAM as a nonstandard POMDP} 
    We formulate active SLAM as a nonstandard POMDP over the joint state space $\mathbb{S} =
    \mathbb{X} \times \mathbb{M}$, where $s_{t}=(x_{t},m_{t})\in \mathbb{S}$,
    by describing the dynamics and observations through stochastic kernels.

    Let $\mathcal{P}(\mathbb{X})$ be the set of probability measures on
    $\mathbb{X}$ endowed with the weak convergence topology. The pose dynamics
    \eqref{eq:robot_dynamics} induce a transition kernel $\mathcal{T}:
    \mathbb{X} \times \mathbb{U} \to \mathcal{P}(\mathbb{X})$ 
    defined by
    \begin{align}
        \mathbb{P}(x_{t+1}\in\cdot\mid x_t=x,\, u_t=u)\!&=\! \mu_{\xi}(\{\xi \in \mathbb{R}^d \!:\! f(x, u, \xi)\in \cdot\,\}) \nonumber\\
        &\coloneqq\mathcal{T}(\, \cdot \mid x,u).
        \label{eq:pose_kernel}
    \end{align}
    The map dynamics are trivially given by the Dirac measure
    \begin{equation*}
        \mathbb{P}(m_{t+1} \in \cdot \mid m_t)\coloneqq \delta_{m_t}(\,\cdot\,)
        =\begin{cases}
            1&\text{if }m_{t+1}=m_t\\0&\text{otherwise.}
        \end{cases}
    \end{equation*}
    The observation kernel $\mathcal{O}: \mathbb{S}\times\mathbb{U} \to \mathcal{P}(\mathbb{Y})$ induced by \eqref{eq:measurement_model} is
    \begin{align}
        \mathbb{P}(y_t\in\cdot\mid s_t=s,u_{t-1}=u)\!&=\!\nonumber\mu_{\zeta}(\{\zeta \in \mathbb{R}^d \!:\! g(s,u, \zeta) \in \cdot\,\})\\
        & \coloneqq\mathcal{O}(\,\cdot \mid s,u).
        \label{eq:observation_kernel}
    \end{align}
    We define active SLAM as a nonstandard $\rho$-POMDP \cite{araya2010pomdp},
    described by $(\mathbb{S}, \mathbb{U}, \mathbb{Y},
    \mathcal{S}, \mathcal{O}, \rho, \beta)$,
    with state
    transition kernel, 
    \begin{equation}
        \mathcal{S}(B_x \times B_m \mid (x_t, m_t), u_t) = 
        \mathcal{T}(B_x \mid x_t, u_t) \delta_{m_t}(B_m)
        \label{eq:augmented_transition}
    \end{equation}
    for sets $B_x \in \mathcal{B}(\mathbb{X})$ and $B_m \in
    \mathcal{B}(\mathbb{M})$, where $\mathcal{B}(\,\cdot\,)$ denotes the Borel
    $\sigma$-algebra, $\rho:\mathcal{P}(\mathbb{S}) \times \mathbb{U} \to
    \mathbb{R}$ is a stage cost dependent on some probability distribution, and
    $\beta \in (0,1)$ is the discount factor.
    The nonstandard aspect of this formulation comes from the fact that the stage cost $\rho$ is dependent on the belief rather than the hidden state as would be the case in a standard POMDP formulation.

    \subsection{Cost function}\label{sub:Cost function} 
        Since the hidden state $s_t$ is not directly observable, decision-making
        conditions on the posterior probability distribution of $s_t$, commonly referred to as the \emph{belief}, denoted by $b_t\in\mathcal{P}(\mathbb{S})$:
        \begin{equation}
            \label{eq:belief}
            b_{t}(\cdot)\coloneq \mathbb{P}(s_{t}\in \cdot \mid
            y_{[0,t]},u_{[0,t-1]})\quad s_{t}\in \mathbb{S},
        \end{equation}
        where $\mathbb{B} \coloneqq \mathcal{P}(\mathbb{S})$ is the belief space.
        The stage cost in active SLAM must balance two competing objectives: (i) \emph{exploration}, rewarding actions that reduce uncertainty about the state being estimated, and (ii) \emph{control effort}, penalizing energy expenditure. We address exploration first.
        
        \subsubsection{Exploration}\label{sub:Exploration} 
        We will first examine the mapping sub-problem for a belief $b_t\in\mathcal{P}(\mathbb{M})$ where we equip the space $\mathbb{M}$ with a metric $d_{\mathbb{M}}$. 
        One can measure distances between probability measures on $\mathbb{M}$ using the Wasserstein metric, defined as
        \begin{equation*}
            W_{1}(\mu,\nu)\coloneq \inf_{\psi\in \mathcal{H}(\mu ,\nu )}
            \int_{\mathbb{M}\times \mathbb{M}} d_{\mathbb{M}}(m,m')\,
            \psi(dm,dm'),
        \end{equation*}
        where $\mathcal{H}(\mu ,\nu )$ denotes the set of couplings of $\mu$ and $\nu$.
        
        A natural measure of estimation quality at time $t$ is the distance from the current belief to the true map, $m^*\in\mathbb{M}$, namely $W_1(b_t,\delta_{m^*})$. Since $m^*$ is unknown, we assess performance by averaging this quantity over possible true maps under a prior $\kappa\in\mathcal{P}(\mathbb{M})$ with $\kappa(\mathcal{U}_{m^*})>0$ for every neighbourhood $\mathcal{U}_{m^*}$ of $m^*$. To formalize this, let $M$ be an $\mathbb{M}$-valued random variable with law $\kappa$, representing the true map. Then
        \begin{equation}
            J_{\text{pa}}(\gamma,\kappa) = \mathbb{E}_{\kappa}^{\gamma}\left[
            \sum_{t=0}^{\infty} \beta^t W_1\bigl(b_t, \delta_M\bigr)\right],
            \label{eq:prior_averaged_cost}
        \end{equation}
        where the subscript $\text{pa}$ stands for \textit{prior averaged}. The per-stage cost admits the closed form
        \begin{equation}
            W_1\bigl(b_t, \delta_m\bigr)
            = \int_{\mathbb{M}} d_{\mathbb{M}}(m', m)\, b_t(dm'),
            \label{eq:wass_stage_cost}
        \end{equation}
        which is the expected distance from the belief to map $m$.

        The formulation \eqref{eq:prior_averaged_cost} is not immediately amenable to dynamic programming because the one-stage term $W_1(b_t,\delta_m)$ depends on the map argument $m$ in addition to the current belief.
        We therefore define a belief-only stage cost,
        \begin{align}
            \tilde W_1(b_t) :&= \int_{\mathbb{M}} W_1(b_t, \delta_m)\,
            b_t(dm)\nonumber\\
            &=\int_{\mathbb{M}\times
            \mathbb{M}}d_{\mathbb{M}}(m,m')b_t(dm)b_t(dm')
            \label{eq:belief_stage_cost}
        .\end{align}
        This functional is identical to the Rao quadratic entropy \cite{rao1982diversity} and is also related to the notion of the potential function of a measure \cite{beiglbock2022approximation}. It is non-negative and $\tilde W_1(b_t)=0$ if and only if $b_t=\delta _{m}$ for some $m\in \mathbb{M}$. We now introduce an equivalent formulation of the objective in \eqref{eq:prior_averaged_cost}: determine an optimal policy $\gamma^{*}\in\Gamma$ that minimizes the following infinite-horizon discounted cost functional with $\rho=\tilde{W}_1$:
        \begin{equation}
            J_{\beta}(\gamma,b_{0}) \coloneq \mathbb{E}^{\gamma}_{b_{0}}\left[
            \sum_{t=0}^{\infty} \beta^t \rho(b_{t})\right].
            \label{eq:infinite_horizon_cost}
        \end{equation}

        \begin{proposition}\label{prop:cost_equivalence}
            For any policy $\gamma$ and prior $\kappa$, with $b_{0} = \kappa$, the prior averaged cost \eqref{eq:prior_averaged_cost} and infinite-horizon objective \eqref{eq:infinite_horizon_cost} with stage cost \eqref{eq:belief_stage_cost} are equivalent:
            $J_{\text{pa}}(\gamma,\kappa) = J_{\beta}(\gamma,b_{0})$.
        \end{proposition}

        \begin{proof}
            Recall that $M$ is an $\mathbb{M}$-valued random variable with law $\kappa$, and
            \begin{equation*}
                b_t(\cdot) = \mathbb{P}(M \in \cdot \mid I_t)
            \end{equation*}
            is the conditional distribution of the true map given the information $I_t$.

            Observe that
            \begin{align*}
                &\mathbb{E}^\gamma_\kappa[W_1(b_t,\delta_M)\mid I_t]=\int_\mathbb{M}W_1(b_t,\delta_m)\mathbb{P}(M\in dm\mid I_t)\\
                &=\int_\mathbb{M}W_1(b_t,\delta_m)b_t(dm)=\tilde{W}_1(b_t).
            \end{align*}
            Then in \eqref{eq:prior_averaged_cost} we have
            \begin{align*}
                J_{\text{pa}}(\gamma,\kappa) &= \mathbb{E}_{\kappa}^\gamma\left[
            \sum_{t=0}^{\infty} \beta^t W_1\bigl(b_t, \delta_M\bigr)\right]\\
           &=\sum_{t=0}^{\infty}\beta^t \mathbb{E}_\kappa^\gamma[W_1(b_t,\delta_M)]\\
            &=\sum_{t=0}^{\infty}\beta^t \mathbb{E}_\kappa^\gamma[\mathbb{E}_\kappa^\gamma[W_1(b_t,\delta_M)\mid I_t]]\\
            &=\sum_{t=0}^{\infty}\beta^t \mathbb{E}_\kappa^\gamma[\tilde{W}_1(b_t)]\\
        &=\mathbb{E}_\kappa^\gamma\left[\sum_{t=0}^{\infty}\beta^t \tilde{W}_1(b_t)\right]=J_\beta(\gamma,b_0).
            \end{align*}
In the second equality above, we utilize the smoothing property of iterated expectations given that $\mathbb{M}$ is compact. The changes in the order of expectations and summations follow from Fubini's theorem as the expressions are non-negative. 

        \end{proof}

        \paragraph*{Relation to prior work.}
        The standard exploration criterion in active SLAM is \textit{information gain} \cite{stachniss2009robotic,carrillo2018autonomous,suresh2024behave}, defined as the expected reduction in Shannon entropy of the belief:
        \begin{align*}
            r_{\text{IG}}(b_{t},u_{t})
            &\coloneq \mathbb{H}(b_{t})-\mathbb{E}_{b_{t+1} \mid
            b_{t},u_{t}}[\mathbb{H}(b_{t+1})].
        \end{align*}
        Julian et al. \cite{julian2013mutual} show that, for the active mapping problem with known pose, this reduces to maximizing mutual information between the map and the next observation.
        Over a finite state space, for the discounted infinite-horizon problem we see that maximizing over a stage cost of $r_{\text{IG}}$ is equivalent to minimizing over $\mathbb{H}(b_t)$.

        $\tilde{W}_1$ offers two advantages over Shannon entropy as an exploration cost. First, it is bounded and weakly continuous over a compact state space, satisfying the measurable selection conditions required for dynamic programming (see Assumption~\ref{ass:6} below). The information gain equivalence above requires a finite state space, also, measurable selection conditions would be satisfied for continuous spaces as differential entropy may be unbounded from below. Second, $\tilde{W}_1$ inherits sensitivity to the ground metric $d_{\mathbb{M}}$, weighting uncertainty across distant hypotheses more heavily than uncertainty among nearby ones. On the other hand, Shannon entropy is permutation invariant and carries no information about distances in the underlying state space.


        While we have restricted our treatment of exploration to the map, the proposed cost can also be understood to encourage localization. Let $b_t\in\mathcal{P}(\mathbb{X})$, then a one-stage cost of the form \eqref{eq:belief_stage_cost} with respect to the metric on $\mathbb{X}$
        penalizes beliefs whose probability mass is spread across distant regions of the state space, thereby encouraging localization. This cost extends naturally to the SLAM setting where the belief is defined over $\mathbb{S}$ by choosing an appropriate product metric.

        \subsubsection{Additional cost criteria} \label{sub:additional cost} 
        Efficiency of actuation is encoded through a control penalty given by
        $c_{\text{effort}}(u) = \lVert u \rVert^{2}$. In occupancy map settings,
        obstacle avoidance can be encouraged via barrier-type penalties
        \cite{marvi2021safeQ}, although existing work only treats unbounded barrier
        functions which would not be compatible with this setup as this would
        violate measurable selection conditions (see Assumption~\ref{ass:6} below).
    \subsection{Belief MDP Reformulation}\label{sub:belief_mdp}
    Recall the belief $b_t$ \eqref{eq:belief}.
    The belief evolves via the nonlinear filter
    $\mathbf{F}:\mathbb{B}\times\mathbb{U}\times\mathbb{Y}\to\mathbb{B}$,
    \begin{align}
        b_{t+1}
        &\coloneq \mathbf{F}(b,u,y)\nonumber\\
        &= \mathbb{P}(s_{t+1}\in\cdot\mid b_t=b,\, u_t=u,\, y_{t+1}=y).
        \label{eq:belief_update}
    \end{align}
    Suppose the observation kernel $\mathcal{O}$ is dominated by a $\sigma$-finite measure $\lambda$ on $\mathbb{Y}$, i.e., $\mathcal{O}(\cdot\mid s,u)\ll \lambda$ for every $(s,u)\in\mathbb{S}\times\mathbb{U}$, and define the likelihood function $o(s,u,y)\coloneq\frac{d\mathcal{O}(\cdot\mid s,u)}{d\lambda}(y)$.
    Then, writing $s'=(x',m')$ and using \eqref{eq:augmented_transition},
    \begin{equation}
        b_{t+1}(ds')
        \!=\! \frac{\int_{\mathbb{X}}o(s',u_t,y_{t+1})
        \mathcal{T}(dx'\mid x,u_t)\,b_t(dx,dm')}
        {\int_{\mathbb{S}} o(\tilde{s},u_t,y_{t+1})
        \int_{\mathbb{X}}\mathcal{T}(d\tilde{x}\mid x,u_t)\,b_t(dx,d\tilde{m})}.
        \label{eq:bayesian_update}
    \end{equation}
    The POMDP reduces to a fully observable Markov decision process (MDP), often referred to as a belief MDP. It is fully defined by the 4-tuple
    $(\mathbb{B}, \mathbb{U}, \eta , \rho)$ with belief transition kernel
    \begin{equation}
        \eta(\,\cdot  \mid b, u) = \int_{\mathbb{Y}} \delta_{\mathbf{F}(b,
        u, y)}(\,\cdot \,)
        \, \mathbf{G}(dy \mid b, u),
        \label{eq:filter_kernel}
    \end{equation}
    where $\mathbf{G}(\,\cdot \mid b, u) \coloneq 
\mathbb{P}(y_{t+1}\in \cdot\mid b_{t}=b, u_{t}=u)$
is the predictive observation distribution,
    and one-stage cost
    \begin{equation}
        \rho(b_t, u_t) = r(b_{t},u_{t}) + \int_{\mathbb{S}} c(s_t, u_t) \, b_t(ds_t),
        \label{eq:expected_cost}
    \end{equation}
    where $c:\mathbb{S}\times \mathbb{U}\to \mathbb{R}$ is a state-dependent
    cost and $r:\mathcal{P}(\mathbb{S})\times \mathbb{U}\to \mathbb{R}$ is a
    belief-dependent cost encouraging exploration.

    We denote admissible control policies by $\gamma=\{ \gamma _{t} \} _{t\ge 0}\in\Gamma$
    where $u_{t}=\gamma_{t}(y_{[0,t]},u_{[0,t-1]})$ and stationary policies by $\gamma^s\in\Gamma^s$.
    The goal is to minimize the infinite horizon discounted cost:
    \begin{equation}
        \tilde{J}(b_{0},\gamma)\coloneq \lim_{T \to \infty}
        \mathbb{E}_{b_{0}}^{\gamma}\left[\sum_{t=0}^{T-1} \beta^{t}
        \rho(b_{t},u_{t}) \right],
        \label{eq:infinite_horizon}
    \end{equation}
    where $\mathbb{E}_{b_{0}}^{\gamma}$ is the expectation with initial
    distribution $s_{0}\sim b_{0}$ under policy $\gamma$.
\section{Properties of Active SLAM}\label{sec:Properties of Active SLAM} 
    We begin by formalizing the Markov structure of the belief process.
    \begin{theorem}[Belief is a controlled Markov chain]
        The belief process $\{ b_{t}, u_{t} \} _{t\ge 0}$ is a controlled
        Markov chain.
        \label{thm:cmc}
    \end{theorem}
    \begin{proof}
        Assume $\mathbb{S},\mathbb{U},\mathbb{Y}$ are standard Borel spaces. Then
        \begin{align*}
        &\mathbb{P}( b_{t+1}\in D\mid  b_{[0,t]},u_{[0,t]})=\mathbb{P}(\mathbf{F}_{ b_{t},Y_{t+1},u_{t}}\in D\mid  b_{[0,t]},u_{[0,t]})\\
        &= \int _{\mathbb{Y}}\mathbb{P}(\mathbf{F}_{ b_{t},y_{t+1},u_{t}}\in D,y_{t+1}\in dy\mid  b_{[0,t]},u_{[0,t]})\\
        &= \int_{\mathbb{Y}}\mathbb{P}(\mathbf{F}_{ b_{t},y_{t+1},u_{t}}\in D\mid y_{t+1}\in dy,  b_{[0,t]},u_{[0,t]})\mathbb{P}(y_{t+1}\in dy\mid  b_{[0,t]},u_{[0,t]})\\
        &= \int_{\mathbb{Y}} \mathbbm{1}_{\{ \mathbf{F}_{ b_{t},y,u_{t}}\in D \}}\mathbb{P}(y_{t+1}\in dy\mid  b_{t},u_{t})\\
        &= \int_{\mathbb{Y}} \mathbbm{1}_{\{ \mathbf{F}_{ b_{t},y,u_{t}}\in D \}} \left( \int_{\mathbb{S}}\int_{\mathbb{S}} \mathbb{P}(y_{t+1}\in dy,s_{t+1}\in ds',s_{t}\in ds\mid  b_{t},u_{t}) \right)\\
        &= \int_{\mathbb{Y}} \mathbbm{1}_{\{ \mathbf{F}_{ b_{t},y,u_{t}}\in D \}} \left( \int_{\mathbb{S}}\int_{\mathbb{S}} Q(dy\mid s')\mathbb{P}(s_{t+1}\in ds',s_{t}\in ds\mid  b_{t},u_{t}) \right)\\
        &= \int_{\mathbb{Y}} \mathbbm{1}_{\{ \mathbf{F}_{ b_{t},y,u_{t}}\in D \}} \left( \int_{\mathbb{S}}\int_{\mathbb{S}} Q(dy\mid s')\mathcal{S}(ds'\mid s, u_{t})\mathbb{P}(s_{t}\in ds\mid  b_{t},u_{t}) \right)\\
        &= \int_{\mathbb{Y}} \mathbbm{1}_{\{ \mathbf{F}_{ b_{t},y,u_{t}}\in D \}} \left( \int_{\mathbb{S}}\int_{\mathbb{S}} Q(dy\mid s')\mathcal{S}(ds'\mid s, u_{t}) b_{t}(ds) \right)\\
        &= \int_{\mathbb{Y}} \mathbbm{1}_{\{ \mathbf{F}_{ b_{t},y,u_{t}}\in D \}} \left( \int_{\mathbb{X}\times \mathbb{M}}\int_{\mathbb{X}\times \mathbb{M}}  Q(dy\mid x',m')\mathcal{T}(dx'\mid x,u_{t})\delta_{m}(dm') b_{t}(dx,dm) \right)\\
        &= \mathbb{P}( b_{t+1}\in D\mid  b_{t},u_{t}).
        \end{align*}
        This along with \cite[Theorem~C.2.1,C.2.2]{yuksel2024} establishes that under
        the weak convergence topology, $\{ b_{t}, u_{t} \}$ forms a Borel controlled Markov chain.
    \end{proof}

    Furthermore, the belief $b_t$ is a sufficient statistic for optimal control \cite{yushkevich1976reduction}, and the stage cost \eqref{eq:expected_cost} is a function of the belief.

    Recall that $\mu_n \to \mu$ weakly if
    $\int h \, d\mu_n \to \int h \, d\mu$ as $n\to\infty$ for every bounded continuous $h$.
    The belief transition kernel $\eta$ has the \emph{weak Feller property} if
    $(b_n, u_n) \to (b, u)$ implies $\eta(\,\cdot \mid b_n, u_n) \to
    \eta(\,\cdot \mid b, u)$ weakly.
    We present the following set of assumptions.
    \begin{assumption}[\cite{feinberg2016partially}]
        \begin{enumerate}
            \item []
            \item The transition probability $\mathcal{T}(\,\cdot \mid x,u)$ is 
                weak Feller.
            \item The observation kernel $\mathcal{O}(\,\cdot \mid x,m,u)$ is continuous
                in total variation.
        \end{enumerate}
        \label{ass:1}
    \end{assumption}
    The following assumptions provide sufficient conditions that are
    straightforward to check for typical robotic systems.
    \begin{assumption}[Sufficient conditions for Assumption \ref{ass:1}]
         A robot is equipped with a motion model $x_{t+1}=f(x_{t},u_{t},\xi_{t})$
         as described in \eqref{eq:robot_dynamics}, that is continuous in
         $(x_{t},u_{t})$ and a measurement model $y_{t}=g(x_{t}, m_{t},
         u_{t})+\zeta_{t}$ that is continuous in $(x_{t}, m_{t}, u_{t})$ and $\zeta_t$ admits a continuous density with respect to some reference measure.
    \label{ass:3}
    \end{assumption}

    \begin{remark}[Verification for common models]
        Assumption~\ref{ass:3} permits multiplicative or state-dependent noise, requiring only that $f(x,u,\xi)$ be continuous in $(x,u)$ for each fixed $\xi$.
        Unicycle \cite{salzmann2020trajectron++} and Dubins airplane \cite{chitsaz2007time} models satisfy Assumption~\ref{ass:3}. When paired with a sensor whose observation dynamics $g(x,m,u) + \zeta$ are continuous in $(x,m,u)$ with continuous density for $\zeta$, such as a range-bearing sensor with additive Gaussian noise, Assumption~\ref{ass:3} holds.
    \end{remark}
    
    We now have the following result,
    \begin{theorem}[Weak Feller property of $\eta$ {\cite{kara2019weak}, \cite{feinberg2016partially}}]
        Provided that any one of Assumptions~\ref{ass:1} or~\ref{ass:3} hold, the belief transition kernel, $\eta$, of \eqref{eq:filter_kernel} is weak Feller.
        \label{thm:weakfeller}
    \end{theorem}
    \begin{proof}
        The hidden state is $s=(x,m)\in\mathbb{S}$, with augmented transition kernel
        \[
            \mathcal{S}(\,\cdot \mid (x,m),u)
            = \mathcal{T}(\,\cdot \mid x,u)\otimes \delta_m,
        \]
        by \eqref{eq:augmented_transition}, and observation channel $\mathcal{O}(\,\cdot \mid x,m,u)$.
        We verify the hypotheses of \cite[Theorem~1]{kara2019weak} on $\mathbb{S}$.
        The observation condition is exactly Assumption~\ref{ass:1}-(2), so it remains
        to show $\mathcal{S}(\,\cdot \mid (x,m),u)$ is weakly continuous in $((x,m),u)$.

        Let $((x_n,m_n),u_n)\to ((x,m),u)$. By Assumption~\ref{ass:1}-(1),
        $\mathcal{T}(\,\cdot \mid x_n,u_n)\to \mathcal{T}(\,\cdot \mid x,u)$ weakly on $\mathbb{X}$.
        Since $m_n\to m$, we have $\delta_{m_n}\to \delta_m$ weakly on $\mathbb{M}$.
        By stability of weak convergence under products \cite[Theorem~2.8]{billingsley1999convergence},
        \[
            \mathcal{S}(\,\cdot \mid (x_n,m_n),u_n)
            = \mathcal{T}(\,\cdot \mid x_n,u_n)\otimes \delta_{m_n}
            \to \mathcal{T}(\,\cdot \mid x,u)\otimes \delta_m
            = \mathcal{S}(\,\cdot \mid (x,m),u)
        \]
        weakly on $\mathbb{S}$, so $\mathcal{S}$ is weak Feller.
        Applying \cite[Theorem~1]{kara2019weak} with hidden state space $\mathbb{S}$
        yields that $\eta$ is weak Feller.
    \end{proof}

    Throughout the rest of this work, we assume the following. 
    \begin{assumption}
    The belief-MDP satisfies the following:
        \begin{enumerate}
            \item The one-stage cost function $\rho:\mathcal{P}(\mathbb{S})\times\mathbb{U}\to\mathbb{R}$ \eqref{eq:expected_cost} is bounded and continuous.
            \item The belief transition $\eta$ is weak Feller.
            \item $\mathcal{P}(\mathbb{S})$ and $\mathbb{U}$ are compact.
        \end{enumerate}
        \label{ass:6}
    \end{assumption}
    \begin{remark}[Verification of assumptions]
        Compactness of $\mathbb{X}$ and
        $\mathbb{U}$ must be enforced through physical constraints as done explicitly in
        Section~\ref{sec:Simulation results}. 
        For $\mathbb{M}_{\ell}=\mathcal{W}^{\bar\ell}$, compactness follows since $\mathcal{W}$ is compact. 
         For $\mathbb{M}_{\text{occ}}$, the Hausdorff metric induces the myopic topology for which compactness follows from \cite{fell1962hausdorff,baddeley1992errors}. 
        Compactness of $\mathbb{S}$ implies $\mathcal{P}(\mathbb{S})$
        is compact with respect to the weak topology.
        Boundedness of the stage cost
        holds for the control penalty on compact
        $\mathbb{U}$ and for \eqref{eq:belief_stage_cost} since
        $\tilde{W}_1(b) \leq \mathrm{diam}(\mathbb{S}) < \infty$ on compact
        $\mathbb{S}$, where $\mathrm{diam}(\cdot)$ denotes diameter. Continuity of $\tilde{W}_1$ in $b$ under weak convergence follows since $b_n \to b$ weakly implies $b_n \otimes b_n \to b \otimes b$ weakly \cite[Theorem~2.8]{billingsley1999convergence}, and $d_{\mathbb{S}}$ is bounded and continuous, so the claim follows by the Portmanteau theorem applied to \eqref{eq:belief_stage_cost}.
    \end{remark}
    
    Up until this point we have spent our time validating the weak Feller property and the above assumptions; these now yield the following optimal policy existence result.
    \begin{theorem}\label{thm:existence}
        Under Assumption~\ref{ass:6}, there exists an optimal
        policy.
    \end{theorem}
    Theorem~\ref{thm:existence} guarantees that an optimal policy exists but does not suggest how to compute one; the belief space $\mathcal{P}(\mathbb{S})$ is infinite-dimensional and the action space is uncountable. The next section develops a finite approximation that makes the problem tractable while preserving near-optimality guarantees.
\section{Finite approximations}\label{sec:Finite Approximations} 
    Since $\mathbb{B} =
    \mathcal{P}(\mathbb{S})$ is infinite-dimensional and $\mathbb{U}$ is typically uncountable, minimizing
    \eqref{eq:infinite_horizon} directly is intractable.
    Following \cite{saldi2019asymptotic} and \cite{kara2025quantobs}, we reduce the
    belief MDP to a finite-state, finite-action, finite-observation MDP.

    \subsection{Approximation Scheme}\label{sub:Belief Space Quantization} 
    We now detail the construction. Since $\mathbb{U}$ is compact, for each $n\ge 1$ there exists a finite $1/n$-net $\mathbb{U}_{n}\subset\mathbb{U}$. Restricting policies to $\mathbb{U}_{n}$ yields vanishing suboptimality as $n\to\infty$ \cite[Theorem~2]{saldi2019asymptotic}.

    For quantization of $\mathbb{Y}$ we impose the following assumption.
    \begin{assumption}[\cite{kara2025quantobs}]
        The observation kernel admits a density
        $\mathcal{O}(dy\mid s)=o(s,y)\,\lambda(dy)$ with $o$
        continuous in $y$ for every $s\in\mathbb{S}$.
        \label{ass:cont_obs}
    \end{assumption}
    Note that Assumption~\ref{ass:3} directly implies this under the additive noise assumption.
    Partition $\mathbb{Y}$ into $n$ disjoint Borel sets
    $\{B_{i}^{\mathbb{Y}}\}_{i=1}^{n}$ with $\bigcup_{i}
    B_{i}^{\mathbb{Y}}=\mathbb{Y}$ and define a finite set of
    representative observations
    $\mathbb{Y}_{n}\coloneq\{y_{1},\ldots,y_{n}\}$ with $y_{i}\in
    B_{i}^{\mathbb{Y}}$. The quantized observation kernel is
    \begin{equation*}
        \mathcal{O}_{n}(y_{i}\mid s)\coloneq
        \mathcal{O}(B_{i}^{\mathbb{Y}}\mid s),\quad i=1,\ldots,n.
    \end{equation*}
    We have that restricting policies to $\mathbb{Y}_n$ yields vanishing 
    suboptimality as $n\to\infty$ \cite[Theorem~5.3]{kara2025quantobs}.

    We equip $\mathbb{S}$ with the product metric
    $d_{\mathbb{S}}((x,m),(x',m'))=(
    d_{\mathbb{X}}(x,x')^{p}+d_{\mathbb{M}}(m,m')^{p})^{1/p}$ for fixed $p \in
    [1,\infty )$. Since $\mathcal{P}(\mathbb{S})$ is compact, it can be
    metrized using the Wasserstein metric $W_{1}$.
    
    For each $n\ge 1$, let $\mathbf{Q}_{n}$ be some lattice quantizer on
    $\mathbb{S}$ such that $d_{\mathbb{S}}(s,\mathbf{Q}_{n}(s)) < \frac{1}{n}$
    for all $s\in \mathbb{S}$. Set $\mathbb{S}_{n}\coloneq
    \mathbf{Q}_{n}(\mathbb{S})$ with $|\mathbb{S}_{n}|=m_{n}$. Since
    $\mathbb{S}$ is compact, then $\mathbb{S}_{n}$ is finite. Then, one can
    approximate any probability measure in $\mathcal{P}(\mathbb{S})$ with
    probability measures in $\mathcal{P}(\mathbb{S}_{n})\coloneq \{ \mu \in \mathcal{P}(\mathbb{S}):\mu(\mathbb{S}_{n})=1 \}$.
    Indeed, for any $\mu \in \mathcal{P}(\mathbb{S})$, we have
    \begin{align}
        \begin{split}
        \inf_{\mu' \in \mathcal{P}(\mathbb{S}_{n})}W_{1}(\mu ,\mu' )
        &\le \inf_{\mathbf{Q}:\mathbb{S}\to \mathbb{S}_{n}}\int_{\mathbb{S}}
        d_{\mathbb{S}}(s,\mathbf{Q}(s)) \,\mu (ds)\\
        &\le \int_{\mathbb{S}} d_{\mathbb{S}}(s,\mathbf{Q}_{n}(s)) \,\mu (ds)
        \le \frac{1}{n}.
        \end{split}\label{eq:ineq1}
    \end{align}
    Now, we define the belief space as $\mathbb{B}_{n}\coloneq
    \mathcal{P}(\mathbb{S}_{n})$. We should note that, since we are now working
    in $\mathbb{S}_{n}$, that this induces some changes to the dynamics of the
    state as well as the stage cost for which their quantized counterparts are now defined for $i\in \mathcal{I}\coloneq\{
    1,...,m_{n} \}$ as
    \begin{align}
        \mathcal{S}_{n}(\,\cdot \mid s_{i}^{n},u)&= \int_{B_{i}^{n}} 
        \mathbf{Q}_{n} * \mathcal{S}(\,\cdot \mid s,u)\nu_{i}^{n}(ds) 
        \label{eq:quantized transition}\\
        c_{n}(s_{i}^{n},u)&=\int_{B_{i}^{n}} c(s,u)\nu_{i}^{n}(ds),
        \label{eq:quantized cost}
    \end{align}
    where $\mathbb{S}_{n}=\{ s_{1}^{n},\ldots,s_{m_{n}}^{n} \} $,
    $B_{i}^{n}=\{ s\in \mathbb{S}:\mathbf{Q}_{n}(s)=s_{i}^{n} \}$, 
    $\nu_{i}^{n}(\cdot)$ is a weighting measure and $\mathbf{Q}_{n}*\mathcal{S}(\cdot \mid
    s,u)\in \mathcal{P}(\mathbb{S}_{n})$ is the pushforward of \eqref{eq:augmented_transition} with respect to the quantizer, that is for $i\in\mathcal{I}$, we have $\mathbf{Q}_{n}*\mathcal{S}(s_{i}^{n}\mid s,u)=\mathcal{S}(B_{i}^{n}\mid s,u)$.
    We define a weighting measure $\nu_{i}^{n}=\delta_{s_{i}^{n}}$; hence
    for equations \eqref{eq:quantized transition}, \eqref{eq:quantized cost}, we
    obtain for $i,j\in \mathcal{I}$ that
    \begin{align*}
        \mathcal{S}_{n}(s_{j}^{n} \mid s_{i}^{n},u) &= \mathcal{S}(B_{j}^{(n)}
        \mid s_{i}^{n},u) \\
        c_{n}(s_{i}^{n},u)&=c(s_{i}^{n},u).
    \end{align*}
    Since the original problem is a $\rho$-POMDP with belief-dependent
    cost \eqref{eq:expected_cost}, the quantized stage cost
    retains the same structure:
    \begin{equation}
        \rho_{n}(b_t,u_t) = r(b_t,u_t) + \int_{\mathbb{S}_{n}} c_{n}(s_t,u_t)\,b_t(ds),
        \quad b_t\in\mathbb{B}_{n},
        \label{eq:rho_n}
    \end{equation}
    where $r$ is the belief-dependent component from
    \eqref{eq:expected_cost} and $c_{n}$ is the quantized
    state-dependent cost from \eqref{eq:quantized cost}.
    This defines a $\rho$-POMDP $(\mathbb{S}_{n},\mathbb{U}_{n},\mathbb{Y}_n,\mathcal{S}_{n},\mathcal{O}_n,\rho_{n},\beta)$,
    for which the equivalent belief-$\text{MDP}_{n}$ is
    $(\mathbb{B}_{n},\mathbb{U}_{n},\eta_{n},\rho_{n})$ where $\eta _{n}$ follows
    the same construction as in
    \eqref{eq:belief_update}-\eqref{eq:filter_kernel}, but with
    $\mathbb{S}_{n}$, $\mathbb{U}_{n}$, $\mathbb{Y}_n$, $\mathcal{S}_{n}$, and $\mathcal{O}_n$ substituted in for
    their non-quantized counterparts.
    
    We note that $\mathbb{B}_{n}$ is a simplex in $\mathbb{R}^{m_{n}}$ since
    $\mathbb{S}_{n}$ is finite with $|\mathbb{S}_{n}|=m_{n}$. This
    fact allows us to use quantization methods as seen in
    \cite{reznik2011algorithm}. In fact, we can use this method to quantize
    $\mathbb{B}_{n}$ in a nearest neighbour manner. In order to achieve this, for each
    $M\ge 1$ we define 
    \begin{equation*}
        \mathbb{B}_{n}^{(M)}\coloneq \left\{ (p_{1},\ldots,p_{m_{n}})\in
            \mathbb{Q}^{m_{n}}:p_{i}= \frac{k_{i}}{M},\, \sum_{i=1}^{m_{n}}
        k_{i}=M \right\},
    \end{equation*}
    where $\mathbb{Q}$ is the set of rational numbers and
    $k_{1},\ldots,k_{m}\in \mathbb{Z}_{+}$. Parameter $M$ serves as a common
    denominator to all fractions, and can be used to control the density and
    number of points in $\mathbb{B}_{n}^{(M)}$. The algorithm that
    computes the nearest neighbour levels is known as Reznik's algorithm \cite[Algorithm~1]{reznik2011algorithm}.
    The size of $\mathbb{B}_{n}^{(M)}$ depends on the size of the
    quantization level $| \mathbb{B}_{n}^{(M)}|= \binom{M+m_{n}-1}{m_{n}-1}$ 
    \cite{reznik2011algorithm}. 
    We now state the following result, which characterizes the error of using the quantized belief MDP:
    \begin{proposition}
    For any $\mu \in \mathbb{B}$ and $\mu_{n}\in  \mathbb{B}_{n}$ such that $\mu _{n} =
    {\arg\min}_{\mu_{n}^{*}\in \mathbb{B}_{n}}\, W_{1}(\mu,\mu_{n}^{*})$, we have
    \begin{equation*}
        \inf_{\mu_{n}^{(M)}\in  \mathbb{B}_{n}^{(M)}}W_{1}(\mu,\mu_{n}^{(M)})\le\frac{1}{n}+ \frac{D}{M} \frac{a(m_{n}-a)}{m_{n}}.
    \end{equation*}
    where $D:=\mathrm{diam}(\mathbb{S})$ and $a=\lfloor m_n/2 \rfloor$.  
    \end{proposition}
    \begin{proof}
    \begin{align*}
    &\inf_{\mu_{n}^{(M)}\in  b_{n}^{(M)}}\!\!\!\!\!W_{1}(\mu,\mu_{n}^{(M)})
    \le W_{1}(\mu,\mu_{n})+\!\!\!\!\!\!\!\inf_{\mu_{n}^{(M)}\in
     b_{n}^{(M)}}W_{1}(\mu_{n},\mu_{n}^{(M)})\\
    &\le \frac{1}{n}+\mathrm{diam}(\mathbb{S}_{n})\cdot \inf_{\mu_{n}^{(M)}\in
     b_{n}^{(M)}}\lVert \mu_{n} -\mu_{n}^{(M)} \rVert _{TV} \\
    &= \frac{1}{n}+\frac{\mathrm{diam}(\mathbb{S}_{n})}{2}\cdot
    \inf_{\mu_{n}^{(M)}\in  b_{n}^{(M)}}\lVert \mu_{n} -\mu_{n}^{(M)} \rVert
    _{1} \\
    &\le \frac{1}{n}+ \frac{D}{2M}\max_{\mu_{n}\in  b_{n}}
    \min_{\mu_{n}^{(M)}\in  b_{n}^{(M)}}\lVert \mu_{n} -\mu_{n}^{(M)} \rVert
    _{1}\\
    &\le \frac{1}{n}+ \frac{D}{M} \frac{a(m_{n}-a)}{m_{n}}
    \end{align*}
    where $D\coloneq \text{diam}(\mathbb{S}_{n})=\sup \{ d_{\mathbb{S}}(x,y) 
    :x,y\in \mathbb{S}\}$. The second inequality is due to \eqref{eq:ineq1} and
    \cite[Theorem~6.15]{villani2008optimal} where for Polish $\mathbb{S}_{n}$
    we have 
    \begin{equation*}
        W_{1}(\mu,\mu')\le \text{diam}(\mathbb{S}_{n})\lVert \mu -\mu' \rVert_{TV}.
    \end{equation*}
    The third line uses the fact that $\mathbb{S}_{n}$ is countable to apply
    \cite[Proposition~4.2]{levin2017markov}: 
    \begin{equation*}
        \lVert \mu -\mu' \rVert_{TV}= \frac{1}{2}\lVert \mu -\mu '\rVert _{1}.
    \end{equation*} 
    \end{proof}
    
    We now define the weighting measure
    $\nu^{(M)}$, where $\nu^{(M)}_{i}(\,\cdot\,)\coloneq \delta_{ b^{(M)}_{i}}(\,\cdot\,)$. We can then define the quantized transition kernel and quantized stage cost for $i,j\in \left\{1,\ldots,\left|\mathbb{B}_n^{(M)}\right|\right\}$:
    \begin{align*}
        \eta_{n}^{(M)}( b_{j}^{(M)} \mid  b_{i}^{(M)} ,u)&\coloneq \eta_{n} (B_{j}^{(M)}
        \mid  b_{i}^{(M)},u)\\
        \rho_{n}^{(M)}(b_{i}^{(M)},u)&\coloneq
        \rho_{n}(b_{i}^{(M)},u).
    \end{align*}
    This defines the belief-MDP$_n^{(M)}$: $(\mathbb{B}_n^{(M)},\mathbb{U}_n,\eta_n^{(M)},\rho_n^{(M)})$.
    \begin{theorem}
        Suppose Assumptions \ref{ass:6} and \ref{ass:cont_obs} hold for the POMDP. 
        Let $(N^n, M^n)_{n \geq 1}$ be a sequence of quantization parameters 
        with $N^n, M^n \to \infty$ as $n \to \infty$.
        Let $\gamma_n$ denote the policy obtained by extending the optimal 
        policy of $\text{MDP}_{N^n}^{(M^n)}$ from $ \mathbb{B}_{N^n}^{(M^n)}$ to 
        $\mathbb{B}_{N^n}$ and again to $\mathbb{B}$. Then
        \begin{align*}
            \lim_{n \to \infty} \left| \tilde{J}(\gamma_n, \mu) - \tilde{J}^{*}(\mu) \right| = 0.
        \end{align*}
        \label{thm:4}
    \end{theorem}
    \begin{proof}
        This result follows directly from
        \cite[Theorem~2]{saldi2019asymptotic},
        ,\cite[Theorem~3]{saldi2019asymptotic} and
        \cite[Theorem~5.3]{kara2025quantobs}.
    \end{proof}
    
\section{Simulation results}\label{sec:Simulation results} 
    In this section, we provide a detailed analysis of a particular setup that satisfies all of the required regularity conditions.
    \subsection{Model description}\label{sub:Model description} 
        We instantiate the active SLAM framework of 
        Sections~\ref{sec:Preliminaries}--\ref{sec:Finite Approximations}
        for a mobile robot equipped with a range-bearing sensor navigating an
        unknown environment.

        We represent the environment using a feature-based map $\mathbb{M} =
        \mathcal{W}^{\ell}$, where $\ell=1$.

        We define the state space as $\mathbb{X} \coloneqq [-L, L]^2 \subset \mathbb{R}^2$ for $L>0$, representing the robot's position in the workspace. Control inputs are constrained to $u_t \in \mathbb{U} \coloneqq \{u\in\mathbb{R}^2:\lVert u\rVert\leq v_{\max}\}$, where $v_{\max} > 0$ is the maximum speed. Both $\mathbb{X}$ and $\mathbb{U}$ are compact, satisfying Assumption~\ref{ass:6}-(3).

        The robot position $x_{t}\in\mathbb{X}$ evolves according to a single-integrator model with sampling period $\Delta t>0$:
        \begin{equation}
            x_{t+1} = \mathsf{P}_{\mathbb{X}}(x_t + u_t \Delta t + w_t),
            \label{eq:double_integrator}
        \end{equation}
        where $w_t \stackrel{iid}{\sim} \mathcal{N}(0, \sigma_w^2 I)$ models process noise, and
        $\mathsf{P}_{\mathbb{X}}$ is a projection operation keeping the state
        inside $\mathbb{X}$.
        Let $\widetilde{\mathcal{T}}(\,\cdot \mid x, u)$ denote the
        unconstrained Gaussian law of $x + u\Delta t + w_t$.
        This kernel is Gaussian with mean $x + u\Delta t$, hence continuous in
        $(x,u)$ in total variation.
        The actual transition kernel induced by
        \eqref{eq:double_integrator} is its pushforward under the projection:
        \begin{equation*}
            \mathcal{T}(\,\cdot \mid x,u) =
            (\mathsf{P}_{\mathbb{X}})_{*}\widetilde{\mathcal{T}}(\,\cdot \mid x,u).
        \end{equation*}
        Therefore $\mathcal{T}$ is also continuous in total variation (hence weakly continuous),
        satisfying Assumptions~\ref{ass:1}-(1).
        Indeed, for the projection
        $\mathsf{P}_{\mathbb{X}}:\mathbb{R}^2\to\mathbb{X}$ and any
        $\mu,\nu\in\mathcal{P}(\mathbb{X})$, for every
        $B\in\mathcal{B}(\mathbb{X})$,
        \begin{align*}
            \lVert (\mathsf{P}_{\mathbb{X}})_{*}\mu -
            (\mathsf{P}_{\mathbb{X}})_{*}\nu \rVert_{\mathrm{TV}}
            &= \sup_{B\in\mathcal{B}(\mathbb{X})}
            \lvert \mu(\mathsf{P}_{\mathbb{X}}^{-1}(B))
            - \nu(\mathsf{P}_{\mathbb{X}}^{-1}(B)) \rvert \\
            &\leq \lVert \mu - \nu \rVert_{\mathrm{TV}},
        \end{align*}
        since $\{\mathsf{P}_{\mathbb{X}}^{-1}(B): B\in\mathcal{B}(\mathbb{X})\}$
        is a subset of the Borel sets on $\mathbb{X}$. Here
        $\mu=\widetilde{\mathcal{T}}(\,\cdot\mid x,u)$, and
        $\nu=\widetilde{\mathcal{T}}(\,\cdot\mid x',u')$.

        Drawing from \cite[Section~6.6.2]{thrun2005probabilistic}, we employ a
        model of a range-bearing sensor. To encode missed detections within
        our SLAM formulation, we augment this with a null observation symbol
        $\perp$ and a continuous detection probability depending on range.

        Let $\mathcal{R} \coloneqq [\epsilon, r_{\max}]$ denote the range
        space with $0 < \epsilon < r_{\max}$ specifying the minimum and maximum
        sensing distances, and let $\Theta \coloneqq [- \pi,  \pi)$ denote the
        bearing space, topologized as $S^1$. For each landmark $i$, define the
        individual measurement space
        $\mathcal{Y}_i \coloneqq (\mathcal{R} \times \Theta) \sqcup \{\perp\}$,
        equipped with the disjoint union topology. The full measurement space is
        $\mathbb{Y} \coloneqq \mathcal{Y}_1 \times \cdots \times \mathcal{Y}_\ell$,
        which is compact under the product topology. An observation at time $t$
        is $y_t = (y_t^1, \ldots, y_t^{\ell})$ where each
        $y_t^i \in \mathcal{Y}_i$ is either a range-bearing pair
        $(r_t^i, \phi_t^i) \in \mathcal{R} \times \Theta$ or the null $\perp$.

        To specify this detection mechanism, define
        $h(z)\coloneqq 3z^{2}-2z^{3}$ on $[0,1]$, and for $a<b$ let
        \begin{equation*}
            \eta_{a,b}(r) \coloneqq
            \begin{cases}
                0, & r \le a,\\
                h\!\left(\dfrac{r-a}{b-a}\right), & a < r < b,\\
                1, & r \ge b.
            \end{cases}
        \end{equation*}
        Choose transition radii $\epsilon < r_{0} < r_{1} < r_{\max}$ and
        define
        \begin{equation*}
            p_{\mathrm{det}}(r)\coloneqq
            \eta_{\epsilon,r_{0}}(r)\bigl(1-\eta_{r_{1},r_{\max}}(r)\bigr).
        \end{equation*}
        Thus
        $p_{\mathrm{det}}(r)=0$ for $r\le \epsilon$ and for $r\ge r_{\max}$,
        while $p_{\mathrm{det}}(r)=1$ on the central interval
        $[r_{0},r_{1}]$. This captures a physically reasonable near-field blind
        zone and a finite sensing horizon, with reliability tapering at
        both boundaries.

        For landmark $i$, write $r^{*,i}=\|m^{i}-x\|$ and
        $\phi^{*,i}=\operatorname{atan2}(m^{i}-x)$. For detected landmarks,
        measurements are governed by the observation function
        $g \colon \mathbb{X} \times \mathbb{M} \times \mathbb{R}^2
        \to \mathcal{R} \times \Theta$,
        \begin{equation*}
            g(x, m^i, \zeta^i) = \begin{pmatrix}
                \lVert m^i - x \rVert + \zeta^{r,i} \\[4pt]
                \mathsf{P}_\Theta\bigl(\operatorname{atan2}(m^i - x) + \zeta^{\phi,i}\bigr)
            \end{pmatrix},
        \end{equation*}
        where $\mathsf{P}_\Theta(\phi) \coloneqq \operatorname{atan2}(\sin\phi,\cos\phi)$
        wraps angles to $\Theta$. The range and bearing noise
        are independent:
        \begin{equation*}
            r_{t}^{i}\sim\mathcal{N}_{[\epsilon,\,r_{\max}]}
            (r^{*,i},\,\sigma_{r}^{2}),\qquad
            \phi_{t}^{i}\sim\mathcal{N}_{\Theta}
            (\phi^{*,i},\,\sigma_{\phi}^{2}),
        \end{equation*}
        where $\mathcal{N}_{[a,b]}$ denotes Gaussian truncated to $[a,b]$, and
        $\mathcal{N}_{\Theta}$ denotes the wrapped normal on
        $\Theta$. Let $o_i(\cdot \mid x,m^i)$ denote the detected-observation kernel induced by
        this model. For $r^{*,i}\le \epsilon$, the
        detected-observation kernel $o_i(\cdot \mid x,m^i)$ may be chosen
        arbitrarily, since it is multiplied by $p_{\mathrm{det}}(r^{*,i})=0$ and
        therefore does not affect the observation law.

        We then define the marginal observation kernel
        $\mathcal{O}_i: \mathbb{X} \times \mathbb{M} \times \mathcal{B}(\mathcal{Y}_i) \to [0,1]$
        by
        \begin{align*}
            \mathcal{O}_i(A_i \mid x, m^i)
            &\coloneqq p_{\mathrm{det}}(r^{*,i}) \,
            o_i(A_i^{\text{obs}} \mid x, m^i)\\
            &\quad + \bigl(1 - p_{\mathrm{det}}(r^{*,i})\bigr)
            \delta_{\perp }(A_i),
        \end{align*}
        where $A_i^{\text{obs}} = A_i \cap (\mathcal{R} \times \Theta)$. In
        particular, when $r^{*,i}\ge r_{\max}$ or $r^{*,i}\le \epsilon$, the
        observation is $\perp$ almost surely. The per-landmark kernels are
        conditionally independent, so the full observation kernel factors as
        \begin{equation*}
            \mathcal{O}(A \mid x, m) = \prod_{i=1}^{\ell} \mathcal{O}_i(A_i \mid x, m^i),\,
            A = \prod_{i=1}^{\ell} A_i \in \mathcal{B}(\mathbb{Y}).
        \end{equation*}
        Because $r^{*,i}$ is continuous in $(x,m^i)$,
        $p_{\mathrm{det}}$ is continuous, and the truncated Gaussian and
        wrapped normal components vary continuously in total variation with their
        parameters, each marginal kernel $\mathcal{O}_i(\cdot \mid x,m^i)$ is
        continuous in total variation. Since $\ell<\infty$, the product kernel
        $\mathcal{O}(\cdot \mid x,m)$ is also continuous in total variation.
        Hence the simulation sensor model satisfies Assumption~\ref{ass:1}-(2),
        and the full simulation setup satisfies Assumption~\ref{ass:1}.
        The stage cost combines the belief dependent costs 
        from Section~\ref{sub:Exploration} and the control penalty
        $\lVert u_{t}\rVert^{2}$ from Section~\ref{sub:additional cost}.
    \begin{figure}[t]
        \centering
        \includegraphics[width=\textwidth]{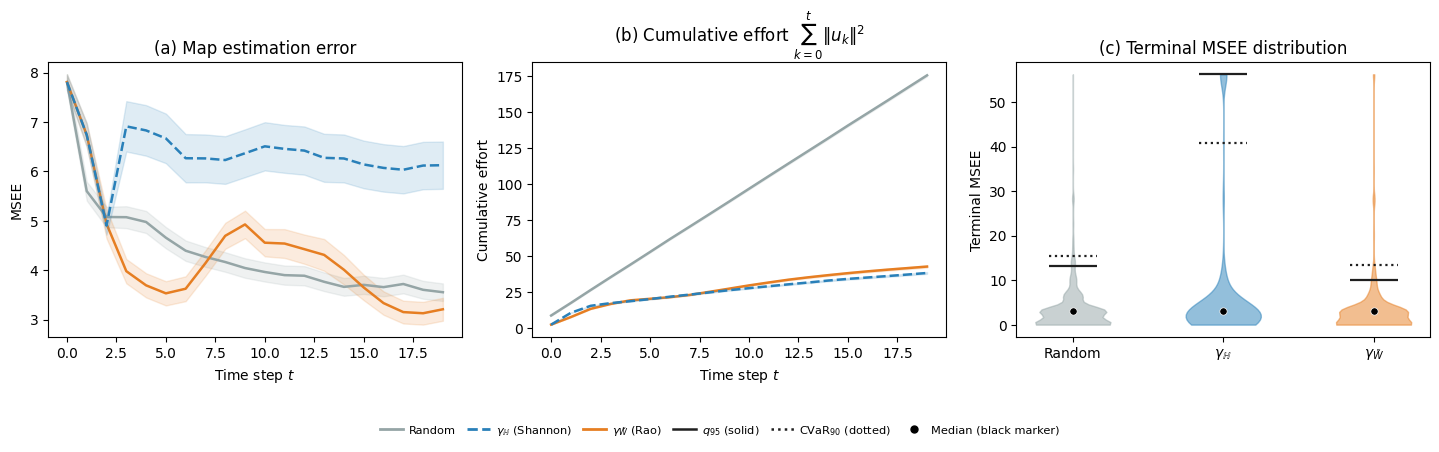}
        \caption{Policy comparison at $M=6$, $\lambda=200$,
        $\sigma_r=0.75$, $\sigma_\phi=0.5$ ($N=3000$ paired trials).
        (a)~Mean MSEE trajectory with 95\% confidence interval.
        (b)~Cumulative control effort.
        (c)~Terminal MSEE distribution with $q_{95}$ and CVaR$_{90}$ annotations.}
        \label{fig:representative}
    \end{figure}

    \subsection{Comparison of exploration criteria}\label{sub:cost_comparison} 
    We test whether the choice of exploration cost (Shannon entropy
    $\mathbb{H}(b_t)$ versus Rao entropy $\tilde{W}_1(b_t)$) affects mapping
    performance. The robot pose is treated as known, so the belief lives on the
    quantized space $\mathbb{B}_n^{(M)} \subset \mathcal{P}(\mathbb{M}_n)$.

    \paragraph{Setup}
    Quantization uses $n=4$ for pose, map, and observation spaces, $n=8$ for
    actions, and $M \in \{5,6\}$ for the belief space. The stage cost is
    \begin{equation*}
        c(b_t,u_t) = \lambda\,\overline{\rho}(b_t) + \lVert u_t \rVert^2,
        \quad b_t \in \mathbb{B}_n^{(M)},\; u_t \in \mathbb{U}_n,
    \end{equation*}
    where $\overline{\rho}$ is the exploration cost normalized to $[0,1]$ and
    $\lambda > 0$ controls the exploration-effort tradeoff. Let
    $\gamma_{\mathbb{H}}$ and $\gamma_{\tilde{W}}$ denote the optimal policies
    obtained by solving \eqref{eq:infinite_horizon} with Shannon and Rao entropy
    respectively. Both policies are computed via value iteration on the same
    quantized belief space and evaluated by the mean squared estimation error (MSEE)
    \begin{equation}
        \mathrm{MSEE}_t \coloneqq
        \mathbb{E}\bigl[\lVert m^* - \hat{m}_t \rVert^2\bigr],
        \label{eq:msee}
    \end{equation}
    where $\hat{m}_t \coloneqq \mathbb{E}_{b_t}[m]$ and
    $m^* \in \mathbb{M}_n$ is the true map.

    For each of 12 noise settings
    ($\sigma_r \in \{0.5, 0.75, 1, 1.25\}$,
    $\sigma_\phi \in \{0.3, 0.5, 0.7\}$),
    $\lambda$ is swept over
    $\{1, 2, 5, 10, 20, 50, 100, 200, 500, 1000, 2000\}$ and selected
    independently for each cost. Performance is compared using the 95th
    percentile of terminal MSEE,
    \begin{equation*}
        q_{95}(\mathrm{MSEE}_t) \coloneqq
        \inf\bigl\{x : \mathbb{P}(\mathrm{MSEE}_t \le x) \ge 0.95\bigr\},
    \end{equation*}
    and the conditional value-at-risk,
    \begin{equation*}
        \mathrm{CVaR}_{90}(\mathrm{MSEE}_t) \coloneqq
        \mathbb{E}\bigl[\mathrm{MSEE}_t
        \mid \mathrm{MSEE}_t \ge q_{90}(\mathrm{MSEE}_t)\bigr].
    \end{equation*}
    Cumulative control effort $\sum_{k=0}^t \lVert u_k \rVert^2$ is also
    tracked. Monte Carlo simulations with $N = 3000$ paired common-random-number
    trajectories and time horizon $T = 20$ are used to approximate the MSEE \eqref{eq:msee} and its
    distributional statistics, with trials balanced over all map hypotheses.

    \paragraph{Results}
    At $M = 5$ ($|\mathbb{B}_n^{(M)}| = 15504$), the two costs produce
    indistinguishable policies across all noise settings and $\lambda$ values.

    At $M = 6$ ($|\mathbb{B}_n^{(M)}| = 54264$), differences emerge.
    \Cref{fig:representative} shows a setting where $\gamma_{\tilde{W}}$
    converges to lower MSEE than both $\gamma_{\mathbb{H}}$ and the random
    baseline at comparable control effort, with a tighter terminal MSEE
    distribution.

    \begin{figure}[t]
        \centering
        \includegraphics[width=\columnwidth]{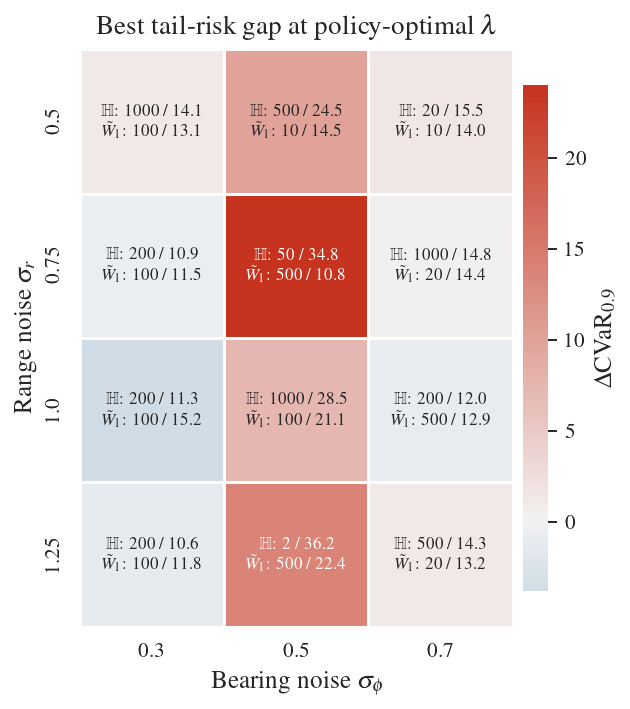}
        \caption{Best $\mathrm{CVaR}_{90}$ at $M=6$, each policy tuned
            independently over $\lambda$. Color shows the gap
            $\mathrm{CVaR}_{90}^{\mathbb{H}} -
            \mathrm{CVaR}_{90}^{\tilde{W}}$, positive favoring
            $\gamma_{\tilde{W}}$. $\gamma_{\tilde{W}}$ attains lower
            $\mathrm{CVaR}_{90}$ in 8 of 12 settings.}
        \label{fig:heatmap}
    \end{figure}

    \Cref{fig:heatmap} summarizes the best achievable
    $\mathrm{CVaR}_{90}$ across all 12 noise settings.
    The results are organized primarily by bearing noise.
    At $\sigma_\phi = 0.3$, $\gamma_{\mathbb{H}}$ performs comparably or
    better. At $\sigma_\phi = 0.5$, $\gamma_{\tilde{W}}$ achieves the largest
    gains, with $\mathrm{CVaR}_{90}$ gaps of 7--24 units. At
    $\sigma_\phi = 0.7$, margins tighten to $\le 1.5$ units.

    $\gamma_{\tilde{W}}$ does not uniformly dominate.
    $\gamma_{\mathbb{H}}$ produces lower MSEE in a majority of individual
    trials. However, $\gamma_{\tilde{W}}$ achieves better tail behavior in 8 of
    12 noise settings. The mean MSEE gap follows the same pattern but at smaller
    magnitude, as $\gamma_{\tilde{W}}$'s fewer but larger wins pull the mean in
    its favor.

    \begin{figure}[t]
        \centering
        \includegraphics[width=\columnwidth]{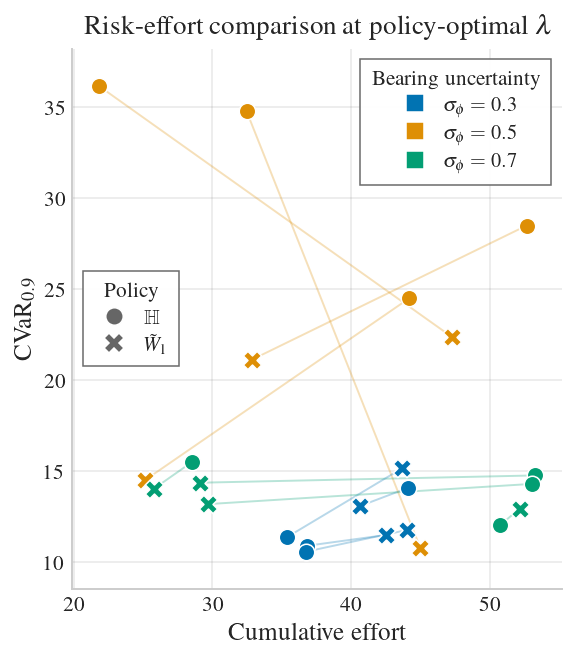}
        \caption{Risk-effort scatter at $M=6$. $\gamma_{\tilde{W}}$ achieves
            both lower $\mathrm{CVaR}_{90}$ and lower effort in 6 of 12
            settings.}
        \label{fig:pareto}
    \end{figure}

    $\gamma_{\tilde{W}}$ achieves its lower tail risk at comparable or reduced
    effort in 6 of 12 settings. In the remaining six,
    $\gamma_{\mathbb{H}}$ is more effort-efficient, and in four of those it
    also achieves lower $\mathrm{CVaR}_{90}$
    (\Cref{fig:pareto}).

    \paragraph{Discussion}
    Bearing noise appears to mediate the tail-risk improvement that $\tilde{W}_1$ produces when compared to the Shannon entropy. The bearing
    channel carries spatial information about landmark geometry, and
    $\gamma_{\tilde{W}}$'s gains are largest at moderate bearing
    noise ($\sigma_\phi = 0.5$). At low bearing noise, beliefs
    concentrate rapidly regardless of exploration strategy, leaving
    little room for geometry-awareness to help. At high bearing
    noise, $\gamma_{\mathbb{H}}$ recovers parity. These experiments
    operate at coarse quantizations ($M \le 6$) where the belief
    space may not fully resolve geometric distinctions, and the
    results should be read as suggestive rather than conclusive.
    
    Future work will investigate approximate solution methods that are less constrained by belief space size, allowing the effect of the exploration cost on policy performance to be evaluated at higher fidelity.


\bibliographystyle{IEEEtran}
\bibliography{refs}

\end{document}